\documentclass{article}
\usepackage{fullpage} 

\usepackage{microtype}
\usepackage{graphicx}
\usepackage{subfigure}
\usepackage{booktabs} 
\usepackage{algorithm, algorithmic}
\usepackage{natbib}
\renewcommand{\cite}{\citep}

\usepackage[utf8]{inputenc} 
\usepackage[T1]{fontenc}    
\usepackage[colorlinks=true,linkcolor=red,citecolor=blue]{hyperref}
\hypersetup{pdfauthor={Name}}
\usepackage{url}            
\usepackage{amsfonts}       
\usepackage{nicefrac}       

\usepackage{amsmath, amsthm, amssymb, mathtools, bbm}

\numberwithin{equation}{section}

\usepackage{commath}
\usepackage[normalem]{ulem}
\usepackage{enumitem}

\def\ddefloop#1{\ifx\ddefloop#1\else\ddef{#1}\expandafter\ddefloop\fi}

\def\ddef#1{\expandafter\def\csname c#1\endcsname{\ensuremath{\mathcal{#1}}}}
\ddefloop ABCDEFGHIJKLMNOPQRSTUVWXYZ\ddefloop

\def\ddef#1{\expandafter\def\csname s#1\endcsname{\ensuremath{\mathsf{#1}}}}
\ddefloop ABCDEFGHIJKLMNOPQRSTUVWXYZ\ddefloop

\def\E{\mathbb{E}}
\def\PP{\mathbb{P}}

\def\Reals{\mathbb{R}}

\def\argmin{\operatornamewithlimits{arg\,min}}

\def\deq{\triangleq}

\def\wh#1{\widehat{#1}}

\def\d{{\mathrm d}}
\def\1{{\mathbf 1}}
\def\lowerG{\iota_G}
\def\lipsG{\kappa_G}
\def\zdom{\mathcal{D}}

\def\eps{\epsilon}

\newtheorem{definition}{Definition}[section]
\newtheorem{theorem}{Theorem}
\newtheorem{lemma}{Lemma}[section]
\newtheorem{proposition}{Proposition}[section]

\usepackage{cleveref}
\usepackage{comment}
\usepackage{tikz}
\usetikzlibrary{calc,angles,positioning,intersections,quotes,decorations.markings}
\usepackage{pgfplots}
\tikzset{
	font={\fontsize{14pt}{12}\selectfont}}

\title{Provable Compressed Sensing with Generative Priors via Langevin Dynamics}

\author{Thanh V. Nguyen, Gauri Jagatap, Chinmay Hegde\thanks{Email: thanhng@iastate.edu, gbj221@nyu.edu, chinmay.h@nyu.edu. This work was partially done while TN was with the Electrical and Computer Engineering Department at Iowa State University. GJ and CH are currently with the Tandon School of Engineering at New York University. This work was supported in part by NSF grants CCF-2005804 and CCF-1815101.}}

\begin{document}
\maketitle

\begin{abstract}
Deep generative models have emerged as a powerful class of priors for signals in various inverse problems such as compressed sensing, phase retrieval and super-resolution. Here, we assume an unknown signal to lie in the range of some pre-trained generative model. A popular approach for signal recovery is via gradient descent in the low-dimensional latent space. While gradient descent has achieved good empirical performance, its theoretical behavior is not well understood. In this paper, we introduce the use of stochastic gradient Langevin dynamics (SGLD) for compressed sensing with a generative prior. Under mild assumptions on the generative model, we prove the convergence of SGLD to the true signal. We also demonstrate competitive empirical performance to standard gradient descent.
\end{abstract}

\section{Introduction}

We consider the familiar setting of inverse problems where the goal is to recover an $n$-dimensional signal $x^*$ that is indirectly observed via a linear measurement operation $y=Ax^*$. The measurement vector can  be noisy, and its dimension $m$ may be less than $n$. Several practical applications fit this setting, including super-resolution \citep{srcnn}, in-painting, denoising \cite{vincent2010stacked}, and compressed sensing \cite{cs,onenet}. 

Since such an inverse problem is ill-posed in general, the recovery of $x^*$ from $y$ often requires assuming a low-dimensional structure or \emph{prior} on $x^*$. Choices of good priors have been extensively explored in the past three decades, including sparsity \cite{bpdn,cosamp}, structured sparsity \cite{modelcs}, end-to-end training via convolutional neural networks \cite{onenet, mousavi2017learning}, pre-trained generative priors \cite{CSGAN}, as well as untrained deep image priors \cite{DIP,netgd}.  
     
In this paper, we focus on a powerful class of priors based on deep generative models. The setup is the following: the unknown signal $x^*$ is assumed to lie in the range of some pre-trained generator network, obtained from (say) a generative adversarial network (GAN) or a variational autoencoder (VAE). That is, $x^* = G(z^*)$ for some $z^*$ in the latent space. The task is again to recover $x^*$ from (noisy) linear measurements. 

Such generative priors have been shown to achieve high empirical success \cite{onenet,CSGAN,deepcs}. However, progress on the theoretical side for inverse problems with generative priors has been much more modest. On the one hand, the seminal work of~\cite{bora2017CS} established the first \emph{statistical} upper bounds (in terms of measurement complexity) for compressed sensing for fairly general generative priors, which was later shown in~\cite{scarlett2020} to be nearly optimal. On the other hand, provable \emph{algorithmic guarantees} for recovery using generative priors are only available in very restrictive cases. The paper~\cite{hand2018global} proves the convergence of (a variant of) gradient descent for shallow generative priors whose weights obey a distributional assumption. The paper~\cite{shah2018solving} proves the convergence of projected gradient descent (PGD) under the assumption that the range of the (possibly deep) generative model $G$ admits a polynomial-time oracle projection. To our knowledge, the most general algorithmic result in this line of work is by~\citet{latorre2019fast}. Here, the authors show that under rather mild and intuitive assumptions on $G$, a linearized alternating direction method of multipliers (ADMM) applied to a regularized mean-squared error loss converges to a (potentially large) neighborhood of $x^*$. 
    
The main barrier for obtaining guarantees for recovery algorithms based on gradient descent is the \emph{non-convexity} of the recovery problem induced by the generator network. Therefore, in this paper we sidestep traditional gradient descent-style optimization methods, and instead show that a very good estimate of $x^*$ can also be obtained by performing stochastic gradient Langevin Dynamics (SGLD) \cite{welling2011sgld, raginsky2017, zhang2017hitting, zou2020faster}. We show that this dynamics amounts to \emph{sampling} from a Gibbs distribution whose energy function is precisely the reconstruction loss \footnote{While preparing this manuscript, we became aware of concurrent work by \citet{jalaldimakis2020} which also pursues a similar Langevin-style approach for solving compressed sensing problems; however, they do not theoretically analyze its dynamics.}.

As a stochastic version of gradient descent, SGLD is simple to implement. However, care must be taken in constructing the additive stochastic perturbation to each gradient update step. Nevertheless, the sampling viewpoint enables us to achieve finite-time convergence guarantees for compressed sensing recovery. To the best of our knowledge, this is the first such result for solving compressed sensing problems with generative neural network priors. Moreover, our analysis succeeds under (slightly) weaker assumptions on the generator network than those made in~\cite{latorre2019fast}. Our specific contributions are as follows:
    \begin{enumerate}[]
    \item We propose a provable compressed sensing recovery algorithm for generative priors based on stochastic gradient Langevin dynamics (SGLD). 
    \item We prove polynomial-time convergence of our proposed recovery algorithm to the true underlying solution, under assumptions of smoothness and near-isometry of $G$. These are technically weaker than the mild assumptions made in~\cite{latorre2019fast}. 
    We emphasize that these conditions are valid for a wide range of generator networks. Section~\ref{sec:results} describes them in greater details. 
    \item We provide several empirical results and demonstrate that our approach is competitive with existing (heuristic) methods based on gradient descent.
    \end{enumerate}
    
\section{Prior work}

We briefly review the literature on compressed sensing with deep generative models. For a thorough survey on deep learning for inverse problems, see~\cite{ongie2020deep}.

In \cite{CSGAN}, the authors provide sufficient conditions under which the solution of the inverse problem is a minimizer of the (possibly non-convex) program:
\begin{align}\label{eq:csprior}
    \min_{x = G(z)}\|Ax-y\|_2^2 \, .
\end{align}
Specifically, they show that if $A$ satisfies the so-called set-Restricted Eigenvalue Condition (REC), then the solution to \eqref{eq:csprior} equals the unknown vector $x^*$. They also show that if the generator $G$ has a latent dimension $k$ and is $L$-Lipschitz, then a matrix $A \in \mathbb{R}^{m \times n}$ populated with i.i.d. Gaussian entries satisfies the REC, provided $m = O(k \log L)$.   However, they propose gradient descent as a heuristic to solve \eqref{eq:csprior}, but do not analyze its convergence. In \cite{shah2018solving}, the authors show that projected gradient descent (PGD) for \eqref{eq:csprior} converges at a linear rate under the REC, but only if there exists a tractable \emph{projection} oracle that can compute $\argmin_z \|x-G(z)\|$ for any $x$. The recent work~\cite{lei2019inverting} provides sufficient conditions under which such a projection can be approximately computed. In \cite{latorre2019fast}, a provable recovery scheme based on ADMM is established, but guarantees recovery only up to a neighborhood around $x^*$.

Note that all the above works assume mild conditions on the weights of the generator, use variations of gradient descent to update the estimate for $x$, and require the forward matrix $A$ to satisfy the REC over the range of $G$. \citet{hand2018global,hand2019global} showed \emph{global} convergence for gradient descent, but under the (strong) assumption that the weights of the trained generator are Gaussian distributed.

Generator networks trained with GANs are most commonly studied. However, \citet{whang2020compressed,asim2019invertible} have recently advocated using \textit{invertible} generative models, which use real-valued non-volume preserving (NVP) transformations \cite{dinh2016density}. An alternate strategy for sampling images consistent with linear forward models was proposed in \cite{lindgren2020conditional} where the authors assume an invertible generative mapping and sample the latent vector $z$ from a second generative invertible prior.

Our proposed approach also traces its roots to Bayesian compressed sensing \cite{ji2007bayesian}, where instead of modeling the problem as estimating a (deterministic) sparse vector, one models the signal $x$ to be sampled from a sparsity promoting {distribution}, such as a Laplace prior. One can then derive the maximum \emph{a posteriori} (MAP) estimate of $x$ under the constraint that the measurements $y=Ax$ are consistent. Our motivation is similar, except that we model the distribution of $x$ as being supported on the range of a generative prior.

\section{Recovery via Langevin dynamics} 
\label{sec:results}

In the rest of the paper, $x\wedge y$ denotes $\min\{x,y\}$ and $x\vee y$ for $\max\{x, y\}$. Given a distribution $\mu$ and set $\cA$, we denote $\mu(\cA)$ the probability measure of $\cA$ with respect to $\mu$. $\|\mu-\nu\|_{TV}$ is the total variation distance between two distributions $\mu$ and $\nu$. Finally, we use standard big-O notation in our analysis.

\subsection{Preliminaries}

We focus on the problem of recovering a signal $x^* \in \Reals^n$ from a set of linear measurements $y \in \Reals^m$ where $$y= Ax^* + \varepsilon.$$ 

To keep our analysis and results simple, we consider zero measurement noise, i.e., $\varepsilon = 0$\footnote{We not in passing that our analysis techniques succeed for any vector $\varepsilon$ with bounded $\ell_2$ norm.}. Here, $A \in \mathbb{R}^{m \times n}$ is a matrix populated with i.i.d. Gaussian entries with mean 0 and variance $1/m$. We assume that $x^*$ belongs to the range of a known generative model $G: \zdom \subset \mathbb{R}^d \rightarrow \Reals^n$; that is, 
$$x^* = G(z^*)~~\text{for some}~~z^* \in \zdom.$$
Following~\cite{bora2017CS}, we restrict $z$ to belong to a $d$-dimensional Euclidean ball, i.e., $\zdom = \mathcal{B}(0, R)$. Then, given the measurements $y$, our goal is to recover $x^*$. Again following~\cite{bora2017CS}, we do so by solving the usual optimization problem:
\begin{equation}
    \label{eqnReconLoss}
    \min_{z \in \zdom} F(z) \deq \| y - AG(z) \|^2.
\end{equation}
Hereon and otherwise stated, $\|\cdot\|$ denotes the $\ell_2$-norm. The most popular approach to solving~\eqref{eqnReconLoss} is to use gradient descent \cite{bora2017CS}. For generative models $G(z)$ defined by deep neural networks, the function $F(z)$ is highly non-convex, and as such, it is impossible to guarantee global signal recovery using regular (projected) gradient descent.

We adopt a slightly more refined approach. Starting from an initial point $z_0 \sim \mu_0$, 
our algorithm computes stochastic gradient updates of the form:
\begin{equation}
\label{eqnSGLD}
    z_{k+1} = z_{k} - \eta \nabla_z F(z) + \sqrt{2\eta \beta^{-1}} \xi_k, \quad k = 0, 1, 2, \dots
\end{equation}
where $\xi_k$ is a unit Gaussian random vector in $\mathbb{R}^d$, $\eta$ is the step size and $\beta$ is an inverse temperature parameter. This update rule is known as \emph{stochastic gradient Langevin dynamics} (SGLD) \cite{welling2011sgld} and has been widely studied both in theory and practice \cite{raginsky2017, zhang2017hitting}. Intuitively,~\eqref{eqnSGLD} is an Euler discretization of the continuous-time {diffusion equation}:
\begin{equation}
\label{eqnConDiffusion}
    \mathrm{d} Z(t) = - \nabla_z F(Z(t)) \mathrm{d}t + \sqrt{2\beta^{-1}} \mathrm{d}B(t), \quad t \geq 0, 
\end{equation}
where $Z(0) \sim \mu_0$. Under standard regularity conditions on $F(z)$, one can show that the above diffusion has a unique invariant Gibbs measure. 

We refine the standard SGLD to account for the boundedness of $z$. Specifically, we require an additional Metropolis-like accept/reject step to ensure that $z_{k+1}$ always belongs to the support $\zdom$, and also is not too far from $z_k$ of the previous iteration. We study this variant for theoretical analysis; in practice we have found that this is not necessary. Algorithm \ref{alg:sgld} (CS-SGLD) shows the detailed algorithm. Note that we can use stochastic (mini-batch) gradient instead of the full gradient $\nabla_z F(z)$.

We wish to derive sufficient conditions on the convergence (in distribution) of the random process in Algorithm \ref{alg:sgld} to the target distribution $\pi$, denoted by:
\begin{equation}
    \pi(\mathrm{d}z) \propto \exp(-\beta F(z))\1(z \in \zdom), \label{eq:gibbs}
\end{equation}
and study its consequence in recovering the true signal $x^*$. This leads to the first guarantees of a stochastic gradient-like method for compressed sensing with generative priors. In order to do so, we make the following three assumptions on the generator network $G(z)$.

\begin{algorithm}[!t]
	\caption{CS-SGLD}
	\label{alg:sgld}
	\begin{algorithmic}
		\STATE \textbf{Input:} step size $\eta$; inverse temperature parameter $\beta$, radius $r$ and Lipschitz constant $L$ of $F(z)$.
	    \STATE Draw $z_0$ from $\mu_0 = \mathcal{N}(0, \frac{1}{2L\beta} \mathbb{I})$ truncated on $\zdom$.
 		\FOR {$k = 0,1,\ldots, $}
 		\STATE Randomly sample $\xi_k \sim \mathcal{N}(0, \mathbb{I})$.
		\STATE $z_{k+1}=z_k-\eta \nabla_z F(z_k) +\sqrt{2\eta/\beta}\xi_k$
		\IF{$z_{k+1} \not \in \mathcal{B}(z_k, r)\cap K$} 
		\STATE $z_{k+1}= z_k$
		\ENDIF
		\ENDFOR 
		\STATE \textbf{Output: $\widehat{z} = z_k$}.
	\end{algorithmic}
\end{algorithm}

\begin{enumerate}[leftmargin=3em]
    \item[{\bf (A.1)}] {\bf Boundedness.} For all $z \in \zdom$, we have that $\| G(z) \| \leq B$ for some $B > 0$. 
    \item[{\bf (A.2)}] {\bf Near-isometry.} $G(z)$ is a near-isometric mapping if there exist $0 < \lowerG \le \lipsG$ such that the following holds for any $z, z' \in \zdom$:
    \begin{align*}
        \lowerG  \| z - z'\| \le \| G(z) - G(z') \| \le \lipsG \| z - z' \|.
    \end{align*}
    \item[{\bf (A.3)}] {\bf Lipschitz gradients.} The Jacobian of $G(z)$ is $M$-Lipschitz, i.e., for any $z, z' \in \zdom$, we have
    \begin{equation*}
        \| \nabla_z G(z) - \nabla_z G(z') \| \le M \| z - z' \|,
    \end{equation*}
    where $\nabla_z G(z) = \frac{\partial G(z)}{\partial z}$ is the Jacobian of the mapping $G(\cdot)$ with respect to $z$.
\end{enumerate}
All three assumptions are justifiable. Assumption {\bf (A.1)} is reasonable due to the bounded domain $K$ and for well-trained generative models $G(z)$ whose target data distribution is normalized. Assumption {\bf (A.2)} is reminiscent of the ubiquitous restricted isometry property (RIP) used for compressed sensing analysis \citep{candes2005RIP} and is recently adopted in \citep{latorre2019fast}. Finally, Assumption {\bf (A.3)} is needed so that the loss function $F(z)$ is smooth, following typical analyses of Markov processes. 

Next, we introduce a new concept of smoothness for generative networks. This concept is a weaker version of a condition on $G(\cdot)$ introduced in \cite{latorre2019fast}.
\begin{definition}[Strong smoothness]
\label{defStrongSmoothness}
The generator network $G(z)$ is $(\alpha,\gamma)$-strongly smooth if there exist $\alpha >0$ and $\gamma \ge 0$ such that for any $z, z' \in \zdom $, we have
\begin{equation}
    \label{eqnStrongSmooth}
    \langle G(z) - G(z'), \nabla_z G(z) (z - z') \rangle \ge \alpha \|z - z' \| ^2 - \gamma.
\end{equation}
\end{definition}
Following~\cite{latorre2019fast} (Assumption 2), we call this property {``strong smoothness''}. However, our definition of strong smoothness requires two parameters instead of one, and is weaker since we allow for an additive slack parameter $\gamma \ge 0$. 

Definition~\ref{defStrongSmoothness} can be closely linked to the following property of the {loss function} $F(z)$ that turns out to be crucial in establishing convergence results for CS-SGLD. 
\begin{definition}[Dissipativity \citep{hale1990asymptotic}] 
\label{defDissipative}
A differentiable function $F(z)$ on $\zdom$ is $(\alpha, \gamma)$-dissipative around $z^*$ if for constants $\alpha >0$ and $\gamma \ge 0$, we have
\begin{equation}
    \label{eqnDissipative}
    \langle z - z^*, \nabla_z F(z) \rangle \ge \alpha \|z - z^* \| ^2 - \gamma.
\end{equation}
\end{definition}
It is straightforward to see that \eqref{eqnDissipative} essentially recovers the strong smoothness condition \eqref{eqnStrongSmooth} if the measurement matrix $A$ is assumed to be the identity matrix. In compressed sensing, it is often the case that $A$ is a (sub)Gaussian matrix and that given a sufficient number of measurements as well as Assumptions {\bf (A.1)}, {\bf (A.2)} and {\bf (A.3)}, the dissipativity of $F(z)$ for such an $A$ can still be established. 

Once $F$ is shown to be dissipative, the machinery of~\cite{raginsky2017, zhang2017hitting, zou2020faster} can be adapted to show that the convergence of CS-SGLD.  The majority of the remainder of the paper is devoted to proving this series of technical claims.


\subsection{Main results}

We first show that a very broad class of generator networks satisfies the assumptions made above. The following proposition is an extension of a result in~\cite{latorre2019fast}. 

\begin{proposition}
    \label{propAssumptionHold}
    Suppose $G(z) : \mathcal{D} \subset \Reals^d \rightarrow \Reals^n$ is a feed-forward neural network with layers of non-decreasing sizes and compact input domain $\mathcal{D}$. Assume that the non-linear activation is a continuously differentiable, strictly increasing function. Then, $G(z)$ satisfies Assumptions \textbf{(A.2)} \& \textbf{(A.3)} with constants $\lowerG, \lipsG, M$, and if $2\lowerG^2 > M \lipsG$, the strong smoothness in Definition \ref{defStrongSmoothness} also holds almost surely with respect to the Lebesgue measure.
\end{proposition}

This proposition merits a thorough discussion. First, architectures with  increasing layer sizes are common; many generative models (such as GANs) assume architectures of this sort. Observe that the non-decreasing layer size condition is much milder than the expansivity ratios of successive layers assumed in related work~\cite{hand2018global,asim2019invertible}.

Second, the compactness assumption of the domain of $G$ is mild, and traces its provenance to earlier related works \cite{bora2017CS, latorre2019fast}. Moreover, common empirical techniques for training generative models (such as GANs) indeed assume that the latent vectors $z$ lie on the surface of a sphere~\cite{white2016sampling}.

Third, common activation functions such as the sigmoid, or the Exponential Linear Unit (ELU) are continuously differentiable and monotonic. Note that the standard Rectified Linear Unit (ReLU) activation does \emph{not} satisfy these conditions, and establishing similar results for ReLU networks is deferred to future work. 

The key for our theoretical analysis, as discussed above, is Definition~\ref{defStrongSmoothness}, and establishing this requires Proposition~\ref{propAssumptionHold}.  Interestingly however, in Section~\ref{sec:exp} below we provide \emph{empirical} evidence that strong smoothness holds for generative adversarial networks with ReLU activation trained on the MNIST and CIFAR-10 image datasets.

We now obtain a measurement complexity result by deriving a bound on the number of measurements required for $F$ to be dissipative.
\begin{lemma}
\label{lmSampleComplexity}
Let $G(z) : \mathcal{D} \subset \Reals^d \rightarrow \Reals^n$ be a feed-forward neural network that satisfies the conditions in Proposition~\ref{propAssumptionHold}. Let $\lipsG$ be its Lipschitz constant. Suppose the number of measurements $m$ satisfies: 
\[
m = \Omega\left(\frac{d}{\delta^2} \log (\lipsG/\gamma) \right) \, ,
\]
for some small constant $\delta > 0$. If the elements of $A$ are drawn according to $\mathcal{N}(0, \frac{1}{m})$, then the loss function $F(z)$ is $(1-\delta, \gamma)$-dissipative with probability at least $1 - \exp(-\Omega(m\delta^2))$.
\end{lemma}

The above result can be derived using covering number arguments, similar to the treatment in~\cite{bora2017CS}. Observe that the number of measurements scales linearly with the dimension of the \emph{latent} vector $z$ instead of the \emph{ambient} dimension, keeping in line with the flavor of results in standard compressed sensing. Recent lower bounds reported~\cite{scarlett2020} also have shown that the scaling of $m$ with respect to $d$ and $\log L$ might be \emph{tight} for compressed sensing recovery in several natural parameter regimes.

We need two more quantities to readily state our convergence guarantee. Both definitions are widely used in the convergence analysis of MCMC methods. %
The first quantity defines the goodness of an initial distribution $\mu_0$ with respect to the target distribution $\pi$. 

\begin{definition}[$\lambda$-warm start, \cite{zou2020faster}]\label{eq:def_warm_start}
Let $\nu$ be a distribution on $\zdom$. An initial distribution $\mu_0$ is $\lambda$-warm start with respect to $\nu$ if
    \begin{align*}
    \sup_{\mathcal{A}: \mathcal{A} \subseteq \zdom } \frac{\mu_0(\mathcal{A})}{\nu(\mathcal{A})}\le \lambda. 
    \end{align*}
\end{definition}

The next quantity is the Cheeger constant that connects the geometry of the objective function and the hitting time of SGLD to a particular set in the domain \cite{zhang2017hitting}. 

\begin{definition}[Cheeger constant]\label{def:cheeger} Let $\mu$ be a probability measure on $\zdom$. We say $\mu$ satisfies the isoperimetric inequality with Cheeger constant $\rho$ if for any $\cA \subset \zdom$,
\begin{align*}
\liminf_{h\rightarrow 0^+} \frac{\mu(\cA_h)-\mu(\cA)}{h}\ge \rho\min\big\{\mu(\cA), 1-\mu(\cA)\big\},
\end{align*}
where $\cA_h = \{u \in K: \exists v \in\cA, \|u - v\|_2\le h\}$.
\end{definition}
  
Putting all the above ingredients together, our main theoretical result describing the convergence of Algorithm~\ref{alg:sgld} (CS-SGLD) for compressed sensing recovery is given as follows.

\begin{theorem}[Convergence of CS-SGLD]
\label{thmMain}
Assume that the generative network $G$ satisfies Assumptions {\bf (A.1)} -- {\bf (A.3)} as well as the strong smoothness condition. 
Consider a signal $x^* = G(z^*)$, and assume that it is measured with $m$ (sub)Gaussian measurements such that $m = \Omega(d \log \lipsG/\gamma)$. Choose an inverse temperature $\beta$ 
and precision parameter $\epsilon > 0$. Then, after $k$ iterations of SGLD in Algorithm \ref{alg:sgld}, we obtain a latent vector $z_k$ such that
\begin{equation}\label{eq:conv}
    \E\left[F(z_k)\right] \leq \eps + O\left(\frac{d}{\beta}\log \left( \frac{\beta}{d}\right)\right),
\end{equation}
    provided the step size $\eta$ and the number of iterations $k$ are chosen such that:
    \[
   \eta = \widetilde{O}\left(\frac{\rho^2 \eps^2}{d^2\beta} \right),~ \text{and} \quad k = \widetilde{O}\left(\frac{d^3\beta^2}{\rho^4\eps^2}\right).
    \]
\end{theorem}
In words, if we choose a high enough inverse temperature and appropriate step size, CS-SGLD converges (in expectation) to a signal estimate with very low loss within a polynomial number of iterations. 

Let us parse the above result further. First, observe that the right hand side of \eqref{eq:conv} consists of two terms. The first term can be made arbitrarily small (at the cost of greater computational cost since $\eta$ decreases ). The second term represents the irreducible expected error of the exact sampling algorithm on the Gibbs measure $\pi(\d z)$, which is worse than the optimal loss obtained at $z=z^*$.

Second, suppose the right hand side of \eqref{eq:conv} is upper bounded by $\epsilon'$. Once SGLD finds an $\epsilon'$-approximate minimizer of the loss, in the regime of sufficient compressed sensing measurements (as specified by Lemma~\ref{lmSampleComplexity}), we can invoke Theorem 1.1 in \cite{bora2017CS} along with Jensen's inequality to immediately obtain a recovery guarantee, i.e.,
$$
\E\left[\norm{x^* - G(z_k)}\right] \leq \sqrt{\epsilon'} .
$$

Third, the convergence rate of CS-SGLD can be slow. In particular, SGLD may require a polynomial number of iterations to recover the true signal, while linearized ADMM~\cite{latorre2019fast} converges within a logarithmic number of iterations up to a \emph{neighborhood} of the true signal. Obtaining an improved characterization of CS-SGLD convergence (or perhaps devising a new linearly convergent algorithm) is an important direction for future work.

Fourth, the above result is for noiseless measurements. A rather similar result can be derived with noisy measurements of bounded noise (says, $\|\varepsilon\| \le \sigma)$. This quantity (times a constant depending on $A$) will affect~\eqref{eq:conv} up to an additive term that scales with $\sigma$. This is precisely in line with most compressed sensing recovery results and for simplicity we omit such a
derivation.


\section{Proof outline}
\label{sec:proof}

In this section, we provide a brief proof sketch of Theorem~\ref{thmMain}, while relegating details to the appendix. At a high level, our analysis is an adaptation of the framework of \cite{zhang2017hitting, zou2020faster} specialized to the problem of compressed sensing recovery using generative priors. The basic ingredient in the proof is the use of conductance analysis to show the convergence of CS-SGLD to the target distribution in total variation distance.

Let $\mu_{k}$ denote the probability measure of $Z_k$ generated by Algorithm \ref{alg:sgld} and $\pi$ denote the target distribution in \ref{eq:gibbs}. The proof of Theorem \ref{thmMain} consists of three main steps:

\begin{enumerate}
\item First, we construct an auxiliary Metropolis-Hasting Markov process to show that $\mu_k$ converges to $\pi$ in total variation for a sufficiently large $k$ and a ``good'' initial distribution $\mu_0$.
\item 
Next, we construct an initial distribution $\mu_0$ that serves as a $\lambda$-warm start with respect to $\pi$.
\item  Finally, we show that a random draw from $\pi$ is a near-minimizer of $F(z)$, proving that CS-SGLD recovers the signal to high fidelity.

\end{enumerate}

We proceed with a characterization of the evolution of the distribution of $z_k$ in Algorithm \ref{alg:sgld}, which basically follows \citep{zou2020faster}.
\subsection{Construction of Metropolis-Hasting SGLD}\label{sec:metropolized_SGLD}
 Let $g(z) = \nabla_z F(z)$, $u$ and $w$ respectively be the points before and after one iteration of Algorithm \ref{alg:sgld}; the Markov chain is written as $u \rightarrow v \rightarrow w$, where $v \sim \mathcal{N}(u - \eta g(u), \frac{2\eta}{\beta} I)$ with the following density:
\begin{align}\label{eq:trans_SGLD}
\begin{split}
&P(v|u)= \bigg[\frac{1}{(4\pi\eta/\beta)^{d/2}}\exp\bigg(-\frac{\|v - u + \eta g(u) \|_2^2}{4\eta/\beta}\bigg)\bigg|u\bigg].
\end{split}
\end{align}%
Without the correction step, $P(v|u)$ is exactly the transition probability of the standard Langevin dynamics. Note also that one can construct a similar density with a stochastic (mini-batch) gradient. The process of $v \rightarrow w$ is
\begin{align}\label{eq:markov_chain_v2w}
w =
    \begin{cases}
    v & v\in\mathcal{B}(u,r)\cap \zdom; \\
    u & \mbox{otherwise}.
    \end{cases}
\end{align}
Let $p(u) = \PP_{v\sim P(\cdot|u)}[v\in\mathcal{B}(u,r)\cap \zdom]$ be the probability of accepting $v$. The conditional density $Q(w|u)$ is 
\begin{align*}
Q(w|u) = (1-p(u))&\delta_{u}(w)
+ P(w|u)\cdot\1\big[w\in\mathcal{B}(u,r)\cap \zdom \big], 
\end{align*}
where $\delta_u(\cdot)$ is the Dirac-delta function at $u$. Similar to \cite{zou2020faster, zhang2017hitting}, we consider the $1/2$-lazy version of the above Markov process, with the  transition distribution
\begin{align}\label{eq:def_trans_lz_sgld}
\mathcal{T}_{u}(w) = \frac{1}{2}\delta_{u}(w) + \frac{1}{2}Q(w|u),
\end{align}
and construct an auxiliary Markov process by adding an extra Metropolis accept/reject step. While proving the ergodicity of the Markov process with transition distribution $\mathcal{T}_u(w)$ is difficult, the auxiliary chain does indeed converge to a unique stationary distribution $\pi \propto e^{-\beta F(z)}\cdot\1(z\in \zdom)$ due to the Metropolis-Hastings correction step.

The auxiliary Markov chain is given as follows: starting from $u$, let $w$ be the state generated from $\mathcal{T}_u(\cdot)$. The Metropolis-Hasting SGLD accepts $w$ with  probability, 
\begin{align*}
\alpha_{u}(w) = \min\bigg\{1, \frac{\mathcal{T}_{w}(u)}{\mathcal{T}_{u}(w)}\cdot\exp\big[-\beta\big(F(w) - F(u)\big)\big]\bigg\}.
\end{align*}
Let $\mathcal{T}^{\star}_u(\cdot)$ denote the transition distribution of the auxiliary Markov process, such that
\begin{align*}
\mathcal{T}^{\star}_u(w) = (1-\alpha_u(w))\delta(u) + \alpha_u(w)\mathcal{T}_u(w).
\end{align*}
Below, we establish the connection between $\mathcal{T}_u(\cdot)$ and $\mathcal{T}^{\star}_u(\cdot)$, as well as the convergence of the original chain in Algorithm \ref{alg:sgld} through a conductance analysis on $\mathcal{T}^{\star}_u(\cdot)$. 

\begin{lemma}
\label{lmAuxiliaryCloseness}
Under Assumptions, $F(z)$ is $L$-smooth and satisfies $\| \nabla_z F(z) \| \le D$ for $z \in \zdom$. For $r=\sqrt{10\eta d/\beta}$, the transition distribution of the chain in Algorithm \ref{alg:sgld}  is $\delta$-close the auxiliary chain, i.e., for any set $\mathcal{A} \subseteq \zdom$
\[
(1-\delta)\mathcal{T}_u^{\star}(\mathcal{A}) \le \mathcal{T}_u(\mathcal{A}) \le (1+\delta)\mathcal{T}_u^{\star}(\mathcal{A}).
\]
where $\delta = 10Ld\eta  +10LD d^{1/2}\beta^{1/2}\eta^{3/2}$.
\end{lemma}
In Appendix  \ref{app:keyLemmas}, we show that $F(z)$ is $L$-smooth with $L=(MB + \lipsG^2)$ and its gradient is bounded by $D=\lipsG^2\|A^\top A\|$.

One can verify that $\mathcal{T}^{\star}_u(\cdot)$ is time-reversible \citep{zhang2017hitting}.  Moreover, following \cite{lovasz1993random, vempala2007geometric}, the convergence of a time-reversible Markov chain to its stationary distribution depends on its conductance, which is defined as follows:

\begin{definition}[Restricted conductance]\label{def:s-conductance}
The conductance of a time-reversible Markov chain with transition distribution $\mathcal{T}^{\star}_u(\cdot)$ and stationary distribution $\pi$ is defined by,
\begin{align*}
\phi \deq \inf_{\cA: \cA\subseteq \zdom, \pi(\cA)\in(0,1)}\frac{\int_{\cA}\mathcal{T}_u(\zdom \backslash \cA) \pi(\d u)}{\min\{\pi(\cA), \pi(\zdom \backslash\cA)\}}.
\end{align*}
\end{definition}

Using the conductance parameter $\phi$ and the closeness $\delta$ between $\mathcal{T}_u(\cdot)$ and $\mathcal{T}^{\star}_u(\cdot)$, we can derive the convergence of $\mathcal{T}_u(\cdot)$ in total variation distance.
\begin{lemma}[\citet{zou2020faster}]
\label{lmApproximateConvergence}
Assume the conditions of Lemma \ref{lmAuxiliaryCloseness} hold. If $\mathcal{T}_u(\cdot)$ is $\delta$-close to $\mathcal{T}^{\star}_u(\cdot)$ with $\delta\le\min\{1-\sqrt{2}/2, \phi/16\}$, and the initial distribution $\mu_0$ serves as a $\lambda$-warm start with respect to $\pi$, then
\begin{align*}
\|\mu_k -\pi\|_{TV}\le \lambda\big(1 - \phi^2/8\big)^k + 16\delta/\phi.    
\end{align*}
\end{lemma}

We will further give a lower bound on $\delta$ in order to establish an explicit convergence rate.
\begin{lemma}[\citet{zou2020faster}]
\label{lmConductanceLB}
Under the same conditions of Lemma \ref{lmAuxiliaryCloseness} and the step size $\eta \le \frac{1}{30Ld} \wedge \frac{d}{25\beta D^2}$, there exists a constant $c_0$ such that
\[
\phi \ge c_0\rho\sqrt{\eta/\beta}.
\]
\end{lemma}

\subsection{Convergence of $\mu_k$ to the target distribution $\pi$}

Armed with these tools, we formally establish the first step of the proof.

\begin{theorem}\label{thmConvergenceInTV}
Suppose that the generative network $G$ satisfies Assumptions {\bf (A.1)} -- {\bf (A.3)} as well as the strong smoothness condition. Set $\eta = O\big(d^{-1} \wedge\rho^2 \beta^{-1} d^{-2}\big)$ and $r = \sqrt{10\eta d/\beta}$, then for any $\lambda$-warm start with respect to $\pi$, the output of Algorithm \ref{alg:sgld} satisfies
\begin{align*}
\|\mu_k- \pi\|_{TV}\le   \lambda(1-C_0\eta)^{k} + C_1\eta^{1/2},
\end{align*}
where $\rho$ is the Cheeger constant of $\pi$, $C_0 = \widetilde{O}\big(\rho^2\beta^{-1}\big)$,  and $C_2 = \widetilde{O}\big(d\beta^{1/2}\rho^{-1}\big)$. In particular, if the step size and the number of iterations satisfy:
\[
   \eta = \widetilde{O}\left(\frac{\rho^2 \eps^2}{d^2\beta} \right),~ \text{and} \quad k = \widetilde{O}\left(\frac{d^2\beta^2\log(\lambda)}{\rho^4\eps^2}\right),
    \]
then $\|\mu_k- \pi\|_{TV} \le \eps$ for $\eps > 0$.
\end{theorem}

The convergence rate is polynomial in the Cheeger constant $\rho$ whose lower bound is difficult to obtain generally. A rough bound $\rho = e^{-\tilde O(d)}$ can be derived using the Poincar\'e constant of the distribution $\pi$, under the smoothness assumption. See \citep{bakry2008simple} for details.

\begin{proof}[Proof outline of Theorem \ref{thmConvergenceInTV}] To prove the result, we find a sufficient condition for $\eta$ that fulfills the requirements of Lemmas \ref{lmAuxiliaryCloseness}, \ref{lmApproximateConvergence} and \ref{lmConductanceLB} hold. For $\eta \le \frac{d}{25\beta D^2}$, we have 
\[
\delta = 10Ld\eta  + 10LD d^{1/2}\beta^{1/2}\eta^{3/2} \le 12Ld\eta.
\]
Moreover, Lemma \ref{lmApproximateConvergence} requires $\delta \le \min\{1-\sqrt{2}/2, \phi/16\}$, while $\phi \ge c_0\rho\sqrt{\eta/\beta}$ by Lemma \ref{lmConductanceLB}, so we can set
\[
\eta = \min \biggl\{\frac{1}{30Ld}, \frac{d}{25\beta D^2} , \frac{c_0^2\rho^2}{(156Ld)^2\beta} \biggr\}
\]
for these conditions to hold. Putting all together, we obtain
\begin{align*}
\|\mu_k - \pi\|_{TV} &\le \lambda\big(1 - \phi^2/8\big)^k + \frac{16\delta}{\phi} \notag\\
&\le \lambda(1- C_0 \eta)^{k} + C_1\eta^{1/2},
\end{align*}
where $C_0 = c_0^2\rho^2/8\beta$, $C_1 = 156Ld\beta^{1/2}\rho^{-1}/c_0$. Therefore, we have proved the first part.

For the second part, to achieve $\epsilon$-sampling error, it suffices to choose $\eta$ and $k$ such that
\begin{align*}
\lambda(1-C_0\eta)^k \le \frac{\epsilon}{2}, ~\text{and} \quad C_1\eta^{1/2} \le \frac{\epsilon}{2}.
\end{align*}
Plugging in $C_0, C_1$ above, we can choose 
\begin{align*}
\eta = O\bigg(\frac{\rho^2\epsilon^2}{d^2\beta}\bigg)\ \text{and} \ k = O\bigg(\frac{\log(\lambda/\epsilon)}{C_0\eta}\bigg) = \widetilde{O}\bigg(\frac{d^2\beta^2\log(\lambda)}{\rho^4\epsilon^2}\bigg)
\end{align*}
such that $\|\mu_k - \pi\|_{TV} \le \eps$, which completes the proof.
\end{proof}

\subsection{Existence of $\lambda$-warm start initial distribution $\mu_0$}
Apart from the step size and the number of iterations, the convergence depends on $\lambda$, the goodness of the initial distribution $\mu_0$. In this part, we specify a particular choice of $\mu_0$ in establish this.

\begin{definition}[Set-Restricted Eigenvalue Condition, \citep{bora2017CS}]
For some parameters $\tau >0$ and $o \ge 0$, $A \in \Reals^{m\times n}$ is called $\text{S-REC}(\tau, o)$ if for all $z, z' \in \zdom$, 
\[
\| A(G(z) - G(z'))\| \ge \tau \| G(z) - G(z') \| - o.
\]
\end{definition}

\begin{lemma}
\label{lmGaussInit}
Suppose that $G(z)$ satisfies the near-isometry property in Assumption \textbf{A.2}, and $F(z)$ is $L$-smooth. If $A$ is $\text{S-REC}(\tau, 0)$, then the Gaussian distribution $\mathcal{N}(0, \frac{1}{2\beta L}\mathbb{I})$ supported on $\zdom$ is a $\lambda$-warm start with respect to $\pi$ with $\lambda = e^{O(d)}$.
\end{lemma}

\begin{proof}
Let $\mu_0$ denote the truncated Gaussian distribution $\mathcal{N}(0, \frac{1}{2\beta L}\mathbb{I})$ on $\zdom$ whose measure is
$$\mu_0(\d z) = {e^{-\beta L\|z\|_2^2} \1(z\in \zdom) \d z}/{\Gamma}$$ 
where $ \Gamma = \int_{\zdom}e^{-\beta L\|z \|_2^2}\d z $ is the normalization constant. 

Along with the target measure $\pi$, we can easily verify that
\begin{align*}
\frac{\mu_0(\d z)}{\pi(\d z)} \le  \frac{\int_{\zdom} e^{-\beta F(z)}\d z}{\Gamma} \cdot e^{-\beta L\|z\|_2^2+\beta F(z)}.
\end{align*}
Our goal is to bound the right hand side. Using the smoothness and the simple fact $F(z^*) = 0$, we have 
\begin{align*}
F(z)\le \frac{L}{2}\|z - z^*\|_2^2 \le L\|z^*\|_2^2 + L\|z\|_2^2,
\end{align*}
which implies that $e^{-\beta L\|z\|_2^2 + \beta F(z)} \le e^{\beta L\|z^*\|_2^2}$.

To bound $\int_{\zdom} e^{-\beta F(z)}\d z$, we use the S-REC property of $A$ as well as the near-isometry of $G(z)$. Recall the objective function:
\begin{align*}
    F(z) &= \| y - AG(z) \|^2 = \| A(G(z) - G(z^*) \|^2
    \ge \tau^2 \|G(z) - G(z^*)\|^2 - o \ge \tau^2 \lowerG^2 \|z - z^*\|^2, 
\end{align*}
where we have dropped $o$ for simplicity. Therefore,
\begin{align*}
    \int_{\zdom} e^{-\beta F(z)}\d z &\le \int_{\zdom} e^{-\beta \tau^2 \lowerG^2 \|z - z^*\|^2}\d z
    \le \left(\frac{\pi}{\beta \tau^2 \lowerG^2}\right)^{d/2}.
\end{align*}

Putting the above results together, we can get
\begin{align*}
\lambda  \le \max_{z\in K}\frac{\mu_0(\d z)}{\pi(\d z)} \le \bigg(\frac{\pi}{\beta \tau^2 \lowerG^2}\bigg)^{d/2} \frac{e^{\beta L\|z^*\|_2^2}}{\Gamma} = e^{O(d)},
\end{align*}
and conclude the proof.
\end{proof}
 
\subsection{Completing the proof}
\label{sec:completeproof}
\begin{proof}[Proof of Theorem \ref{thmMain}]
Consider a random draw $\wh{Z}$ from $\mu_{k}$ and another $\wh{Z}^*$ from  $\pi$. We have
\begin{align*}
	\E[F(\wh{Z})]
& = \left( \E [F(\wh{Z})] - \E [F(\wh{Z}^*)] \right) + \E [F(\wh{Z}^*)].
\end{align*}
We will first give a crude bound for the second term $\E [F(\wh{Z}^*)]$ following the idea from \citep{raginsky2017}:
\begin{align*}
	\E [F(\wh{Z}^*)] = \int_{\zdom} F(z) \pi(\d z) 
	&\le \cO\left(\frac{d}{\beta}\log \frac{\beta}{d}\right) .
\end{align*}
The detailed proof is given in Appendix \ref{app:MarkovSemi}. The first term is related to the convergence of $\mu_k$ to $\pi$ in total variation shown in Theorem \ref{thmConvergenceInTV}. Notice that $F(z) \le 2R\|A\|\lipsG $ for all $z \in \zdom$ due the Lipschitz property of the generative network $G$. Moreover, by Theorem \ref{thmConvergenceInTV}, we have $\|\mu_k - \pi\|_{TV} \le \eps'$ for any $\eps' >0$ and a sufficiently large $k$. Hence, the first term is upper bounded by
\begin{align*}
     &\left\vert \int_{\zdom} F(z) \mu_{k}(\d z) - \int_{\zdom} F(z) \pi(\d z)  \right\vert 
    \le 2R\|A\|\lipsG \left\vert \int_{\zdom} \mu_{k}(\d z) - \int_{\zdom} \pi(\d z) \right\vert
    \le 2R\|A\|\lipsG \eps'.
\end{align*}

Given the target error $\eps$, choose $\eps' = {\eps}/{(2R\|A\|\lipsG)}$. By Lemma \ref{lmGaussInit}, we have $\lambda = e^{O(d)}$.
Then, for
\[
  \eta = \widetilde{O}\left(\frac{\rho^2 \eps^2}{d^2\beta} \right),~ \text{and} \quad k = \widetilde{O}\left(\frac{d^3\beta^2}{\rho^4\eps^2}\right), ~\text{we have}
\] 
\[
\E[F(\wh{Z})] \le \eps + \cO\left(\frac{d}{\beta}\log \frac{\beta}{d}\right).
\]
Therefore, we complete the proof of our main result. 
\end{proof}

\section{Experimental results}
\label{sec:exp}

While we emphasize that the primary focus of our paper is theoretical, we corroborate our theory with representative experimental results on MNIST and CIFAR-10. Even though we require requires a bounded domain within $d$-dimensional Euclidean ball in our analysis, empirical results demonstrate that our approach works without the restriction.

\begin{figure}[h!]
    \centering
    \begin{tabular}{ccc}
        \raisebox{1.3\height}{\includegraphics[width=0.1\linewidth]{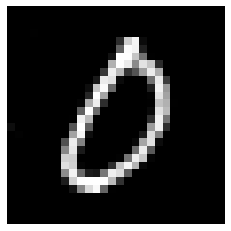}} &
        \includegraphics[width=0.42\linewidth]{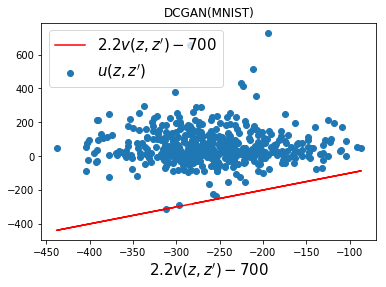} &
        \includegraphics[width=0.42\linewidth]{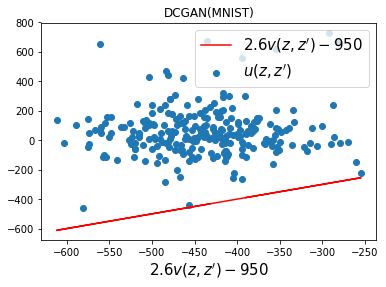}      \\ 
        & (a) & (b) \\
        \raisebox{1.3\height}{\includegraphics[width=0.1\linewidth]{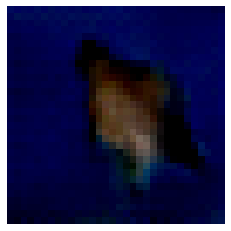}} &
        \includegraphics[width=0.42\linewidth]{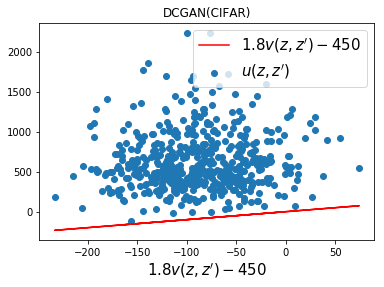} 
        &
        \includegraphics[width=0.42\linewidth]{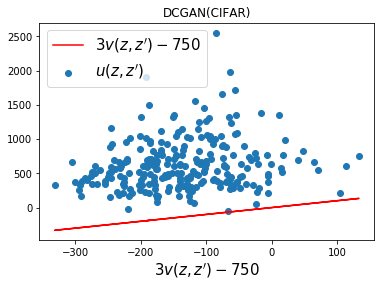}
        \\
      & (c) & (d)
    \end{tabular}
    \caption{[MNIST] selected base digit $G(z^*)$, evaluating (a) dissipativity (b) \eqref{eqnDissipative4-app}, [CIFAR] selected base image $G(z^*)$, evaluating (c) dissipativity (d)  \eqref{eqnDissipative4-app}.}
    \label{fig:bounds-app} 
\end{figure}

\subsection{Validation of strong smoothness}
As mentioned above, our theory relies on the assumption that the following condition holds for some constants $\alpha > 0, \gamma \ge 0$ and $\forall z, z' \in \zdom$ for a domain $\zdom$. Here, we take $\zdom = \Reals^d$.  
\begin{align*}
\langle G(z)-G(z'), \nabla_z G(z)(z-z') \rangle \geq \alpha \|z-z'\|^2 - \gamma.
\end{align*}

To estimate these constants, we generate samples $z$ and $z'$ from $\mathcal{N}(0,\mathbb{I})$. To establish $\alpha$ and $\gamma$, we perform experiments on two different datasets (i) MNIST (Net1) and (ii) CIFAR10 (Net2). For both datasets, we compute the terms $u(z,z') = \nabla_{z} G(z)^{\top}(G(z)-G(z')),z-z' \rangle$ and $v(z,z') = \|z-z'\|^2$ for 500 different instantiations of $z$ and $z'$. We then plot these pairs of $(\alpha v - \gamma,u)$ samples for different $z$'s and $z'$'s and compute  the values of $\alpha$ and $\gamma$ by a simple linear program. We do this experiment for two DCGAN generators trained on MNIST (Figure \ref{fig:bounds-app} (a)) as well as on CIFAR10 (Figure \ref{fig:bounds-app} (c)).

Similarly for the compressed sensing case, we also derive values $\alpha_A$ and $\gamma_A$, where a compressive matrix $A$ acts on the output of the generator $G$. Here, we have picked the number of measurements $m = 0.1n$ where $n$ is the signal dimension. This is encapsulated in the following equation:
\begin{align}
    \langle \nabla_z (AG(z))^{\top}(AG(z)-AG(z')),z-z' \rangle \geq \alpha_A \|z-z'\|^2 - \gamma_A 
    \label{eqnDissipative4-app}
\end{align}
for all possible Gaussian matrices $A$ and different instantiations of $z$ and $z'$. Here, we capture the left side of the inequality in $u(z,z') = \langle \nabla_{z} (AG(z))^{\top}(AG(z)-AG(z')),z-z' \rangle$. 
We similarly plot points $(\alpha_A v(z, z') -\gamma_A, u(z, z'))$ for all pairs $(z, z')$. The scatter plot is generated for 500 different instantiations of $z$ and $z'$ and $5$ different instantiations of $A$.  We do this experiment for two DCGAN generators, one trained on MNIST (Figure \ref{fig:bounds-app} (b)) and the other trained on CIFAR10 (Figure \ref{fig:bounds-app} (d)). These experiments indicate that the dissipativity constant $\alpha$ is positive in all cases.

\subsection{Comparison of SGLD against GD}
\begin{figure*}[!ht]
     \centering
     \begin{tabular}{cccc}
         \includegraphics[width=0.2\linewidth]{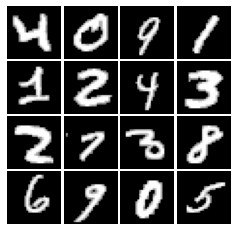} &       \includegraphics[width=0.2\linewidth]{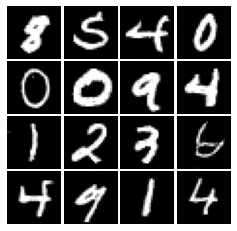} &    \includegraphics[width=0.2\linewidth]{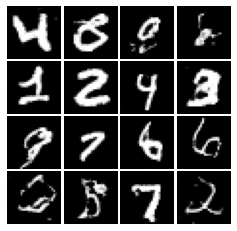} &          \includegraphics[width=0.2\linewidth]{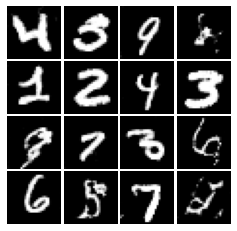} \\
         (a) Ground truth & (b) Initial & (c) GD  & (d) SGLD \\
         &  & MSE = 0.0447 & MSE = 0.0275
     \end{tabular}
     \caption{Comparing the recovery performance of SGLD and GD at $m=0.2n$ measurements. \label{fig:reconstructions}}
 \end{figure*}
 
We test the SGLD reconstruction by using the update rule in \eqref{eqnSGLD} and compare against the optimizing the updates of $z$ using standard gradient descent as in \cite{bora2017CS}.  For all experiments, we use a pre-trained DCGAN generator, with network configuration described as follows: the generator consists of four different layers consisting of transposed convolutions, batch normalization and RELU activation; this is followed by a final layer with a transposed convolution and $\tanh$ activation \cite{DCGAN}. 
 
We display the reconstructions on MNIST in Figure \ref{fig:reconstructions}. Note that the implementation in \cite{bora2017CS} requires 10 random restarts for CS reconstruction and they report the results corresponding to the best reconstruction. This likely suggests that the standard implementation is likely to get stuck in bad local minima or saddle points. For the sake of fair comparison, we fix the same random initialization of latent vector $z$ for both GD and SGLD with no restarts. We select $m=0.2n$. In Figure \ref{fig:reconstructions} we show reconstructions for the 16 different examples, which were all reconstructed at once using same $k=2000$ steps, learning rate of $\eta=0.02$ and the inverse temperature $\beta=1$ for both approaches. The only difference is the additional noise term in SGLD (Figure \ref{fig:reconstructions} part (d)). Notice that this additional noise component helps achieve better reconstruction performance overall as compared to simple gradient descent.

Phase transition plots scanning a range of compression ratios $m/n$ as well as example reconstructions on CIFAR-10 images can be found in the supplement. More thorough empirical comparisons with PGD-based approaches~\cite{shah2018solving,raj2019gan} are deferred to future work.

\subsection{Reconstructions for CIFAR10}

We display the reconstructions on CIFAR10 in Figure \ref{fig:reconstructions2}. As with the implementation for MNIST, for the sake of fair comparison, we fix the same random initialization of latent vector $z$ for both GD and SGLD with no restarts. We select $m=0.3n$. In Figure \ref{fig:reconstructions2} we show reconstructions for the 16 different examples from MNIST, which were all reconstructed at once using same $k=2000$ steps, learning rate of $\eta=0.05$ and the inverse temperature $\beta=1$ for both approaches. The only difference is the additional noise term in SGLD (Figure \ref{fig:reconstructions} part (d)). Similar to our experiments on MNIST we notice that this additional noise component helps achieve better reconstruction performance overall as compared to simple gradient descent. 
 
 \begin{figure}[!h]
     \centering
     \begin{tabular}{cccc}
         \includegraphics[width=0.23\linewidth]{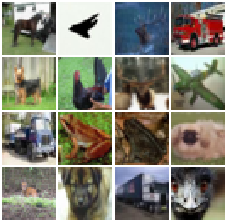} &               \includegraphics[width=0.23\linewidth]{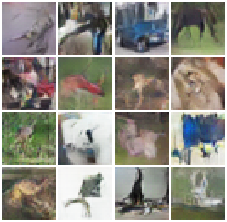} &    \includegraphics[width=0.23\linewidth]{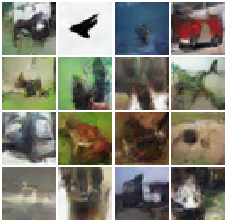} &          \includegraphics[width=0.23\linewidth]{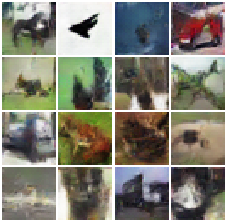} \\
         (a) Ground truth & (b) Initial & (c) GD  & (d) SGLD \\
         &  & MSE = 0.0248 & MSE = 0.0246
     \end{tabular}
     \caption{ [CIFAR10] Comparing the recovery performance of SGLD and GD at $m=0.3n$ measurements.} \label{fig:reconstructions2}
 \end{figure}
 
 Next, we plot phase transition diagrams by scanning the compression ratio $f=m/n = [0.2,0.4,0.6,0.8,1.0]$ for the MNIST dataset in Figure \ref{fig:phasetrans}. For this experiment, we have chosen 5 different instantiations of the sampling matrix $A$ for each compression ratio $f$. In Figure \ref{fig:phasetrans} we report the average Mean Square Error (MSE) of reconstruction $\|\hat{x}-x\|^2$ over 5 different instances of $A$.
 \begin{figure}[!h]
     \centering
     \begin{tikzpicture}[scale=0.5] 
\begin{semilogyaxis}
			[width=0.75\textwidth,
			xlabel= Compression ratio $f$, 
			ylabel= $\log$ MSE,
			grid style = dashed,
			grid=both,
			y label style={at={(axis description cs:(-0.1,0.5)},anchor=north},
			legend style=
			{at={(0.7,0.95)}, 
				anchor=south west, 
				anchor= north , 
			} ,
			]
	
			\addplot[color=magenta, solid, mark=*, mark size = 2, line width =2] plot coordinates {
				(0.2,     0.1464) 
				(0.4,  0.0732)
				(0.6,  0.0461)
				(0.8,     0.0429)
				(1.0,    0.0436)
				
			};
			
				\addplot[color=blue, solid, mark=square, line width=2] plot coordinates {
				(0.2,     0.1434) 
				(0.4,  0.0679)
				(0.6,  0.0451)
				(0.8,     0.0395)
				(1.0,    0.0404)
			};

			\legend{GD\\SGLD\\}
			\end{semilogyaxis}
		\end{tikzpicture}
     \caption{Phase transition plots representing average MSE of reconstructed image using gradient descent and stochastic gradient Langevin dynamics.}
     \label{fig:phasetrans}
 \end{figure}
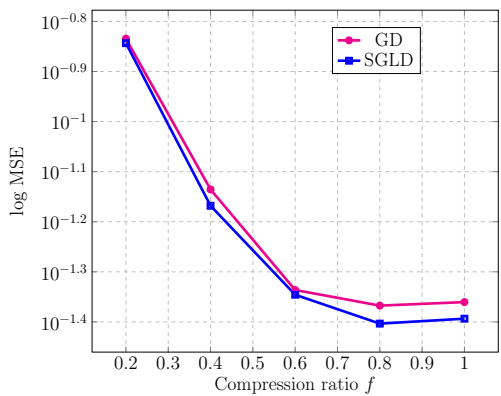
 We conclude that SGLD gives improved reconstruction quality as compared to GD. 

\clearpage

\bibliographystyle{icml2021}
\bibliography{refs}

\begin{thebibliography}{39}
\providecommand{\natexlab}[1]{#1}
\providecommand{\url}[1]{\texttt{#1}}
\expandafter\ifx\csname urlstyle\endcsname\relax
  \providecommand{\doi}[1]{doi: #1}\else
  \providecommand{\doi}{doi: \begingroup \urlstyle{rm}\Url}\fi

\bibitem[Asim et~al.(2019)Asim, Ahmed, and Hand]{asim2019invertible}
Asim, M., Ahmed, A., and Hand, P.
\newblock Invertible generative models for inverse problems: mitigating
  representation error and dataset bias.
\newblock \emph{arXiv preprint arXiv:1905.11672}, 2019.

\bibitem[Bakry et~al.(2008)Bakry, Barthe, Cattiaux, Guillin,
  et~al.]{bakry2008simple}
Bakry, D., Barthe, F., Cattiaux, P., Guillin, A., et~al.
\newblock A simple proof of the poincar{\'e} inequality for a large class of
  probability measures.
\newblock \emph{Electronic Communications in Probability}, 13:\penalty0 60--66,
  2008.

\bibitem[Baraniuk et~al.(2010)Baraniuk, Cevher, Duarte, and Hegde]{modelcs}
Baraniuk, R., Cevher, V., Duarte, M., and Hegde, C.
\newblock Model-based compressive sensing.
\newblock \emph{IEEE Transactions on Information Theory}, 56:\penalty0
  1982--2001, 2010.

\bibitem[Bora et~al.(2017{\natexlab{a}})Bora, Jalal, Price, and Dimakis]{CSGAN}
Bora, A., Jalal, A., Price, E., and Dimakis, A.
\newblock Compressed sensing using generative models.
\newblock In \emph{Proceedings of the 34th International Conference on Machine
  Learning-Volume 70}, pp.\  537--546. JMLR. org, 2017{\natexlab{a}}.

\bibitem[Bora et~al.(2017{\natexlab{b}})Bora, Jalal, Price, and
  Dimakis]{bora2017CS}
Bora, A., Jalal, A., Price, E., and Dimakis, A.~G.
\newblock Compressed sensing using generative models.
\newblock In \emph{Proceedings of the 34th International Conference on Machine
  Learning-Volume 70}, pp.\  537--546. JMLR. org, 2017{\natexlab{b}}.

\bibitem[Candes \& Tao(2005)Candes and Tao]{candes2005RIP}
Candes, E.~J. and Tao, T.
\newblock Decoding by linear programming.
\newblock \emph{IEEE transactions on information theory}, 51\penalty0
  (12):\penalty0 4203--4215, 2005.

\bibitem[Chang et~al.(2017)Chang, Li, P{\'o}czos, and Kumar]{onenet}
Chang, J., Li, C., P{\'o}czos, B., and Kumar, B.
\newblock One network to solve them all—solving linear inverse problems using
  deep projection models.
\newblock In \emph{2017 IEEE International Conference on Computer Vision
  (ICCV)}, pp.\  5889--5898. IEEE, 2017.

\bibitem[Chen et~al.(2001)Chen, Donoho, and Saunders]{bpdn}
Chen, S., Donoho, D., and Saunders, M.
\newblock Atomic decomposition by basis pursuit.
\newblock \emph{SIAM review}, 43\penalty0 (1):\penalty0 129--159, 2001.

\bibitem[Dinh et~al.(2017)Dinh, Sohl-Dickstein, and Bengio]{dinh2016density}
Dinh, L., Sohl-Dickstein, J., and Bengio, S.
\newblock Density estimation using real nvp.
\newblock \emph{International Conference on Learning Representations}, 2017.

\bibitem[Dong et~al.(2016)Dong, Loy, He, and Tang]{srcnn}
Dong, C., Loy, C., He, K., and Tang, X.
\newblock Image super-resolution using deep convolutional networks.
\newblock \emph{IEEE transactions on pattern analysis and machine
  intelligence}, 38\penalty0 (2):\penalty0 295--307, 2016.

\bibitem[Donoho(2006)]{cs}
Donoho, D.
\newblock Compressed sensing.
\newblock \emph{IEEE Transactions on information theory}, 52\penalty0
  (4):\penalty0 1289--1306, 2006.

\bibitem[Hale(1990)]{hale1990asymptotic}
Hale, J.
\newblock Asymptotic behavior of dissipative systems.
\newblock \emph{Bull. Am. Math. Soc}, 22:\penalty0 175--183, 1990.

\bibitem[Hand \& Voroninski(2018)Hand and Voroninski]{hand2018global}
Hand, P. and Voroninski, V.
\newblock Global guarantees for enforcing deep generative priors by empirical
  risk.
\newblock In \emph{Conference On Learning Theory}, pp.\  970--978, 2018.

\bibitem[Hand \& Voroninski(2019)Hand and Voroninski]{hand2019global}
Hand, P. and Voroninski, V.
\newblock Global guarantees for enforcing deep generative priors by empirical
  risk.
\newblock \emph{IEEE Transactions on Information Theory}, 66\penalty0
  (1):\penalty0 401--418, 2019.

\bibitem[Jagatap \& Hegde(2019)Jagatap and Hegde]{netgd}
Jagatap, G. and Hegde, C.
\newblock Algorithmic guarantees for inverse imaging with untrained network
  priors.
\newblock In \emph{Advances in Neural Information Processing Systems}, 2019.

\bibitem[Jalal et~al.(2020)Jalal, ECE, Karmalkar, CS, Dimakis, and
  Price]{jalaldimakis2020}
Jalal, A., ECE, U., Karmalkar, S., CS, U., Dimakis, A.~G., and Price, E.
\newblock Compressed sensing with approximate priors via conditional
  resampling.
\newblock Preprint, 2020.

\bibitem[Ji \& Carin(2007)Ji and Carin]{ji2007bayesian}
Ji, S. and Carin, L.
\newblock Bayesian compressive sensing and projection optimization.
\newblock In \emph{Proceedings of the 24th international conference on Machine
  learning}, pp.\  377--384, 2007.

\bibitem[Latorre et~al.(2019)Latorre, Eftekhari, and Cevher]{latorre2019fast}
Latorre, F., Eftekhari, A., and Cevher, V.
\newblock Fast and provable admm for learning with generative priors.
\newblock In \emph{Advances in Neural Information Processing Systems}, pp.\
  12004--12016, 2019.

\bibitem[Lee \& Vempala(2018)Lee and Vempala]{lee2018convergence}
Lee, Y.~T. and Vempala, S.~S.
\newblock Convergence rate of riemannian hamiltonian monte carlo and faster
  polytope volume computation.
\newblock In \emph{Proceedings of the 50th Annual ACM SIGACT Symposium on
  Theory of Computing}, pp.\  1115--1121, 2018.

\bibitem[Lei et~al.(2019)Lei, Jalal, Dhillon, and Dimakis]{lei2019inverting}
Lei, Q., Jalal, A., Dhillon, I.~S., and Dimakis, A.~G.
\newblock Inverting deep generative models, one layer at a time.
\newblock In \emph{Advances in Neural Information Processing Systems}, pp.\
  13910--13919, 2019.

\bibitem[Lindgren et~al.(2020)Lindgren, Whang, and
  Dimakis]{lindgren2020conditional}
Lindgren, E.~M., Whang, J., and Dimakis, A.~G.
\newblock Conditional sampling from invertible generative models with
  applications to inverse problems.
\newblock \emph{arXiv preprint arXiv:2002.11743}, 2020.

\bibitem[Liu \& Scarlett(2020)Liu and Scarlett]{scarlett2020}
Liu, Z. and Scarlett, J.
\newblock Information-theoretic lower bounds for compressive sensing with
  generative models.
\newblock \emph{IEEE Journal on Selected Areas in Information Theory}, 2020.

\bibitem[Lov{\'a}sz \& Vempala(2007)Lov{\'a}sz and
  Vempala]{vempala2007geometric}
Lov{\'a}sz, L. and Vempala, S.
\newblock The geometry of logconcave functions and sampling algorithms.
\newblock \emph{Random Structures \& Algorithms}, 30\penalty0 (3):\penalty0
  307--358, 2007.

\bibitem[Lov{\'a}sz et~al.(1993)]{lovasz1993random}
Lov{\'a}sz, L. et~al.
\newblock Random walks on graphs: A survey.
\newblock \emph{Combinatorics, Paul erdos is eighty}, 2\penalty0 (1):\penalty0
  1--46, 1993.

\bibitem[Mousavi \& Baraniuk(2017)Mousavi and Baraniuk]{mousavi2017learning}
Mousavi, A. and Baraniuk, R.
\newblock Learning to invert: Signal recovery via deep convolutional networks.
\newblock In \emph{2017 IEEE international conference on acoustics, speech and
  signal processing (ICASSP)}, pp.\  2272--2276. IEEE, 2017.

\bibitem[Needell \& Tropp(2009)Needell and Tropp]{cosamp}
Needell, D. and Tropp, J.
\newblock Cosamp: Iterative signal recovery from incomplete and inaccurate
  samples.
\newblock \emph{Applied and computational harmonic analysis}, 26\penalty0
  (3):\penalty0 301--321, 2009.

\bibitem[Ongie et~al.(2020)Ongie, Jalal, Metzler, Baraniuk, Dimakis, and
  Willett]{ongie2020deep}
Ongie, G., Jalal, A., Metzler, C., Baraniuk, R., Dimakis, A., and Willett, R.
\newblock Deep learning techniques for inverse problems in imaging.
\newblock \emph{arXiv preprint arXiv:2005.06001}, 2020.

\bibitem[Radford et~al.(2015)Radford, Metz, and Chintala]{DCGAN}
Radford, A., Metz, L., and Chintala, S.
\newblock Unsupervised representation learning with deep convolutional
  generative adversarial networks.
\newblock \emph{arXiv preprint arXiv:1511.06434}, 2015.

\bibitem[Raginsky et~al.(2017)Raginsky, Rakhlin, and Telgarsky]{raginsky2017}
Raginsky, M., Rakhlin, A., and Telgarsky, M.
\newblock Non-convex learning via stochastic gradient langevin dynamics: a
  nonasymptotic analysis.
\newblock \emph{arXiv preprint arXiv:1702.03849}, 2017.

\bibitem[Raj et~al.(2019)Raj, Li, and Bresler]{raj2019gan}
Raj, A., Li, Y., and Bresler, Y.
\newblock Gan-based projector for faster recovery with convergence guarantees
  in linear inverse problems.
\newblock In \emph{Proceedings of the IEEE International Conference on Computer
  Vision}, pp.\  5602--5611, 2019.

\bibitem[Shah \& Hegde(2018)Shah and Hegde]{shah2018solving}
Shah, V. and Hegde, C.
\newblock Solving linear inverse problems using gan priors: An algorithm with
  provable guarantees.
\newblock In \emph{2018 IEEE International Conference on Acoustics, Speech and
  Signal Processing (ICASSP)}, pp.\  4609--4613. IEEE, 2018.

\bibitem[Ulyanov et~al.(2018)Ulyanov, Vedaldi, and Lempitsky]{DIP}
Ulyanov, D., Vedaldi, A., and Lempitsky, V.
\newblock Deep image prior.
\newblock In \emph{Proceedings of the IEEE Conference on Computer Vision and
  Pattern Recognition}, pp.\  9446--9454, 2018.

\bibitem[Vincent et~al.(2010)Vincent, Larochelle, Lajoie, Bengio, and
  Manzagol]{vincent2010stacked}
Vincent, P., Larochelle, H., Lajoie, I., Bengio, Y., and Manzagol, P.
\newblock Stacked denoising autoencoders: Learning useful representations in a
  deep network with a local denoising criterion.
\newblock \emph{Journal of machine learning research}, 11\penalty0
  (Dec):\penalty0 3371--3408, 2010.

\bibitem[Welling \& Teh(2011)Welling and Teh]{welling2011sgld}
Welling, M. and Teh, Y.~W.
\newblock Bayesian learning via stochastic gradient langevin dynamics.
\newblock In \emph{Proceedings of the 28th international conference on machine
  learning (ICML-11)}, pp.\  681--688, 2011.

\bibitem[Whang et~al.(2020)Whang, Lei, and Dimakis]{whang2020compressed}
Whang, J., Lei, Q., and Dimakis, A.
\newblock Compressed sensing with invertible generative models and dependent
  noise.
\newblock \emph{arXiv preprint arXiv:2003.08089}, 2020.

\bibitem[White(2016)]{white2016sampling}
White, T.
\newblock Sampling generative networks.
\newblock \emph{arXiv preprint arXiv:1609.04468}, 2016.

\bibitem[Y.~Wu(2019)]{deepcs}
Y.~Wu, M.~Rosca, T.~L.
\newblock Deep compressed sensing.
\newblock \emph{arXiv preprint arXiv:1905.06723}, 2019.

\bibitem[Zhang et~al.(2017)Zhang, Liang, and Charikar]{zhang2017hitting}
Zhang, Y., Liang, P., and Charikar, M.
\newblock A hitting time analysis of stochastic gradient langevin dynamics.
\newblock In \emph{Conference on Learning Theory}, pp.\  1980--2022. PMLR,
  2017.

\bibitem[Zou et~al.(2020)Zou, Xu, and Gu]{zou2020faster}
Zou, D., Xu, P., and Gu, Q.
\newblock Faster convergence of stochastic gradient langevin dynamics for
  non-log-concave sampling.
\newblock \emph{arXiv preprint arXiv:2010.09597}, 2020.

\end{thebibliography}

\clearpage
\appendix

\section{Conditions on the generator network}
\label{app:gen}

\begin{proposition}
    \label{propAssumptionHold-app}
    Suppose $G(z) : \mathcal{D} \subset \Reals^d \rightarrow \Reals^n$ is a feed-forward neural network with layers of non-increasing sizes and compact input domain $\mathcal{D}$. Assume that the non-linear activation is a continuously differentiable, strictly increasing function. Then, $G(z)$ satisfies Assumptions \textbf{(A.2)} \& \textbf{(A.3)} with constants $\lowerG, \lipsG, M$, and if $2\lowerG^2 > M \lipsG$, the strong smoothness in Definition \ref{defStrongSmoothness} also holds almost surely with respect to the Lebesgue measure.
\end{proposition}
\begin{proof} 
The proof proceeds similar to \cite{latorre2019fast}, Appendix B. Since $G(z)$ is a composition of linear maps followed by $C^1$ activation functions, $G(z)$ is continuously differentiable. As a result, the Jacobian $\nabla_z G$
is a continuous matrix-valued function and its restriction
to the compact domain $\mathcal{D} \subseteq \Reals^d$ is Lipschitz-continuous. Therefore, there exists
$M \geq 0$ such that
\begin{equation}
    \label{eqnJacobianLipschitz}
    \| \nabla_zG(z) - \nabla_z G(z') \| \leq M \|{z-z'}\|, \qquad \forall z, z' \in \mathcal{D}.
\end{equation}
Thus, Assumption \textbf{(A.3)} holds. Assumption \textbf{(A.2)} is also satisfied according to \cite{latorre2019fast}, Lemma 5. To show the strong smoothness, we use the fundamental theorem of calculus with the Lipchitzness of $G(z)$ obtained by Assumption \textbf{(A.2)}. For every $z,z'\in \mathcal{D}$, and $u(t) = tz + (1-t)z'$:
\begin{align*}
\langle G(z) - G(z'),  &\nabla_z G(z)(z-z') \rangle \\
&=  \| G(z) - G(z') \|^2 - \langle G(z) - G(z'), G(z) - G(z') - \nabla_z G(z)(z-z') \rangle \\
&= \| G(z) - G(z') \|^2 - \int_{0}^1 \langle G(z) - G(z'), \bigl(\nabla_z G(u(t)) - \nabla_z G(z)\bigr)(z-z') \rangle \d t \\
&\geq \lowerG^2 \| z - z' \|^2 - \lipsG M  \|z - z' \|^2 \int_{0}^1 (1-t)  \d t \\
&= (\lowerG^2 - \frac{\lipsG M}{2}) \| z - z' \|^2,
\end{align*}
where in the last step we use the near-isometry and the Lipschitzness of $\nabla_z G(z)$ we have obtained. Consequently, $G(z)$ is $(\lowerG^2 - \frac{\lipsG M}{2}, 0)$-strongly smooth, if $\lowerG ^2 > \frac{\lipsG M}{2}$.
\end{proof}

\begin{lemma}[Measurement complexity]
\label{lmSampleComplexity-app}

Let $G(z) : \mathcal{D} \subset \Reals^d \rightarrow \Reals^n$ be a feed-forward neural network that satisfies the conditions in Proposition~\ref{propAssumptionHold}. Let $L$ be its Lipschitz constant. If the number of measurements $m$ satisfies: 
\[
m = \Omega\left(\frac{d}{\delta^2} \log (\lipsG/\gamma) \right) \, ,
\]
for some small constant $\delta > 0$. If the elements of $A$ are drawn according to $\mathcal{N}(0, \frac{1}{m})$, then the loss function $F(z)$ is $(\alpha-\delta \lipsG^2, \gamma)$-dissipative with probability at least $1 - \exp(-\Omega(m\delta^2)$.
\end{lemma}

\begin{proof}
Using Proposition \ref{propAssumptionHold-app}, it follows that there exist $\alpha > 0$ and $\gamma \ge 0$ such that $G(z)$ is strongly smooth. Now, note that the left hand side of~\eqref{eqnDissipative} is simplified as
\begin{align}
\label{eqnFzDissipative}
\langle z - z^*, \nabla_z F(z) \rangle &= \left\langle  A (G(z) - G(z^*)), A \nabla_z G(z) (z - z^*) \right\rangle,
\end{align}
Denote $u = G(z) - G(z^*)$ and $v = \nabla_z G(z) (z - z^*)$, then 
\[
\langle z - z^*, \nabla_z F(z) \rangle = \langle Au ,  A v \rangle = \langle u, v\rangle - \langle (\mathbb{I}-A^\top A) u, v\rangle.
\]
Using standard result in random matrix theory, we can get $P(\| \mathbb{I}-A^\top A\| \ge \delta) \le \exp(-m\delta^2)$. Also, $\|u\|, \|v\| \le \lipsG\|z - z'\|$. Therefore,
\[
\langle z - z^*, \nabla_z F(z) \rangle \ge \langle u, v\rangle - \delta \| z - z'\|^2.
\]
For $m = \Omega\left(\frac{d}{\delta^2} \log (\lipsG/\gamma) \right)$, then
\[
\langle z - z^*, \nabla_z F(z) \rangle \ge (\alpha - \delta) \| z - z'\| - \gamma,
\]
with probability at least $1 - \exp(-\Omega(m\delta^2)$. Therefore, the loss function $F(z)$ is $(\alpha-\delta \lipsG^2, \gamma)$-dissipative with probability at least $1 - \exp(-\Omega(m\delta^2)$.
\end{proof}

\section{Properties of $F(z)$}
\label{app:keyLemmas}

In this part, we establish some key properties of the loss function $F(z)$. We use Assumptions {\bf (A.1)} -- {\bf (A.3)} on the boundedness, Lipschitz gradient and near-isometry to obtain an upper bound of $\| \nabla_z F(z) \|$ and the smoothness of $F(z)$.

\begin{lemma}[Lipschitzness of $F(z)$]
\label{lmGradientBound}
We have $\| \nabla_z F(z) \| \le \lipsG^2 \| A^\top A \| \| z - z^* \|$ for any $z \in \zdom \subset \Reals^d$. 
\end{lemma}
\begin{proof}
Recall the gradient of $F(z)$:
\begin{align*}
  \nabla_z F(z) &= -(\nabla_z G(z))^\top A^\top (y - AG(z))
  = -(\nabla_z G(z))^\top A^\top A (G(z^*) - G(z)).
\end{align*}

It follows from the Lipschitz assumption {\bf  (A.2)} that $\| G(z^*) - G(z) \| \le \lipsG \| z - z^* \|$, and hence $\| \nabla_z G(z) \| \le \lipsG$. Therefore,
\[
\| \nabla_z F(z) \| \leq \lipsG^2\| A^\top A \| \| z - z^* \|.
\]

\end{proof}

\begin{lemma}[Smoothness of $F(z)$] \label{lmSmoothness} For any $z, z' \in \zdom \subset \Reals^d$, we have
\[
\| \nabla_z F(z) - \nabla_z F(z') \| \le (MB + \lipsG^2)\| A^\top A \| \| z - z'\|.
\]
\end{lemma}
\begin{proof}
We use the assumptions on $G(z)$ to derive the bound: $\| G(z^*) \| \le B$.
\begin{align*}
\| \nabla_z F(z) - \nabla_z F(z') \| &\leq \| (\nabla_z G(z') - \nabla_z G(z))^\top A^\top A G(z^*) \|  \\
&\qquad + \| (\nabla_z G(z))^\top A^\top A (G(z) - G(z')) \|  \\ 
&\qquad + \| (\nabla_z G(z) - \nabla_z G(z'))^\top A^\top A G(z') \|
\end{align*}
Then, using the boundedness, Lipschitzness and smoothness, we arrive at:
\[
\| \nabla_z F(z) - \nabla_z F(z') \| \leq (MB + \lipsG^2) \| A^\top A \|.
\]
Therefore, $F(z)$ is $L$-smooth, with $L=(MB + \lipsG^2) \| A^\top A \|$.
\end{proof}

\section{Conductance Analysis}
In this section, we provide the proofs of Lemma  \ref{lmAuxiliaryCloseness} and \ref{lmConductanceLB} based on the conductance analysis laid out in \citep{zhang2017hitting} and similarly in \citep{zou2020faster}. The proof of \ref{lmApproximateConvergence} directly follows from Lemma 6.3, \citep{zou2020faster}.

\begin{proof}[Proof of Lemma \ref{lmAuxiliaryCloseness}]
We use the same idea in Lemma 3 from \citep{zhang2017hitting} and similarly that in Lemma 6.1 from \citep{zou2020faster}.
The main difference of our proof is that we use full gradient $\nabla_z F(z)$ in Algorithm \ref{alg:sgld}, instead of stochastic mini-batch gradient, which simplifies the proof of this lemma a little.

We consider two cases for each $u$: $u \not\in \cA$ and $u\in\cA$. As long as we can prove the first case, the second case easily follows, by splitting $\cA$ into $\{u\}$ and $\cA \backslash \{u\}$ and using the result of the first case. For a detailed treatment of the latter case, we refer the reader to the proof of Lemma 6.1 in \citep{zou2020faster}.

Now that $u \notin \cA$, we have
\begin{align}\label{eq:relation_Q2P_case1}
    \mathcal{T}^{\star}_{u}(\cA)=\int_{\cA \cap \mathcal{B}(u, r)} \mathcal{T}^{\star}_{u}(w)\d w = \int_{\cA \cap \mathcal{B}(u, r) } \alpha_u(w) \mathcal{T}_{u}(w)\d w.
\end{align}
where $\alpha_u(w)$ is the acceptance ratio of the Metropolis-Hasting. If suffices to show that $\alpha_u(w)\ge 1-\delta/2$ for all $w\in K \cap \cB(u,r)$, which implies
\begin{align*}
(1-\delta/2)\mathcal{T}_u(\cA)\le \mathcal{T}^{\star}_u(\cA)\le \mathcal{T}_u(\cA).
\end{align*}
The right hand side is obvious by the definition of $\alpha_u(w)$ while we can ensure $\delta \le 1/2$ with a sufficiently small $\eta$. What remains is to show that
\begin{align}\label{eq:lowerbound_accept}
\frac{\mathcal{T}_{w}(u)}{\mathcal{T}_u(w)}\cdot\exp(-\beta(F(w) - F(u)))\ge 1-\delta/2.
\end{align}
The left hand side is simplified by definition of $\mathcal{T}_u(w)$ as
\begin{align*}
\exp\bigg(\frac{\|w - u + \eta g(u)\|_2^2}{4\eta/\beta} - \frac{\|u - w + \eta g(w)\|_2^2}{4\eta/\beta}\bigg) \exp(-\beta(F(w) - F(u))) \ge 1- \delta/2.
\end{align*}
Note that $g(z) = \nabla_z F(z)$. Simplify the first exponent and combine with the second one gives the following form:
\begin{align}
\label{eq:simplified_exponent}
-\beta \left(F(w) - F(u) - \frac{1}{2}\langle w-u, \nabla_z F(w) + \nabla_z F(u) \rangle \right) + \frac{\eta\beta}{4} (\|\nabla_z F(u)\|^2 - \|\nabla_z F(w)\|^2). 
\end{align}
To lower bound the left hand side, we appeal to the smoothness of $F(z)$. Specifically, by Lemmas \ref{lmGradientBound} and \ref{lmSmoothness}, we have $F$ is $L$-smooth and $\| \nabla_z F(z) \| \le D$ with $L=(MB + \lipsG^2)$ and $D=\lipsG^2\|A^\top A\|$. Then,
\begin{align*}
    F(w) &\leq F(u) +\langle w - u, \nabla F(u) \rangle+\frac{L\|w -u\|_2^2}{2},\\
    F(u) &\geq F(w) +\langle u - w, \nabla F(w)\rangle-\frac{L\|w - u\|_2^2}{2}.
\end{align*}
This directly implies that
\begin{align}
\big|F(w) - F(u) -\langle w - u, \frac{1}{2}\nabla F(w) + F(u)\rangle\big|\le \frac{L\|w - u\|_2^2}{2}.\label{eq:Nestorov}
\end{align}
Moreover, 
\begin{align}\label{eq:upperbound_variance_difference}
\big|\|\nabla_z F(u)\|_2^2 - \|\nabla_z F(w)\|_2^2\big|&\leq \|\nabla F(u)-\nabla F(w)\|_2\cdot\|\nabla F(u)+\nabla F(w)\|_2\notag\\
&\le 2LD\|w-u\|_2.
\end{align}
Combining \eqref{eq:Nestorov} and \eqref{eq:upperbound_variance_difference} in \eqref{eq:simplified_exponent}, and together with $w \in \mathcal{B}(u, r)$ with $r = \sqrt{10\eta d/\beta}$,
\begin{align*}
\text{LHS of \eqref{eq:simplified_exponent}} &\ge -\frac{L\beta \| w -u\|^2}{2} - \frac{\eta\beta LD \|w - u\|}{2} \\ 
&\ge -5Ld\eta - 5LG d^{1/2}\beta^{-1/2}\eta^{3/2}.
\end{align*}
Pick $\delta/2 = 5Ld\eta + 5LD d^{1/2}\beta^{-1/2}\eta^{3/2}$, and use the fact $e^{-x} \ge 1 - x$ for $x \ge 0$, then we have proved the result.
\end{proof}

Next, we lower bound the conductance $\phi$ of $\mathcal{T}^{\star}_u(\cdot)$ using the idea in \citep{lee2018convergence, zou2020faster}, by first restating the following lemma: 
\begin{lemma}[Lemma 13 in \citet{lee2018convergence}]\label{lemma:lowerbound_transitionprob}
Let $\mathcal{T}^{\star}_u(\cdot)$ be a time-reversible Markov chain on $\zdom$ with stationary distribution $\pi$. Suppose for any $u,v\in \zdom$ and a fixed $\Delta> 0$ such that $\|u-v\|_2\le \Delta$, we have $\|\mathcal{T}^{\star}_u(\cdot) - \mathcal{T}^{\star}_v(\cdot)\|_{TV}\le 0.99$, then the conductance of $\mathcal{T}^{\star}_u(\cdot)$ satisfies $\phi\ge C\rho\Delta$ for some constant $C > 0$ and $\rho$ is the Cheeger constant of $\pi$.
\end{lemma}

\begin{proof}[Proof of Lemma \ref{lmConductanceLB}] To apply Lemma \ref{lemma:lowerbound_transitionprob}, we follow the same idea of \citet{zou2020faster} and reuse some of their results without proof. To this end, we prove that for some $\Delta$, any pair of $u, v \in \zdom$ such that $\|u - v\|_2\le \Delta$, we have $\|\mathcal{T}^{\star}_u(\cdot) - \mathcal{T}^{\star}_v(\cdot)\|_{TV}\le 0.99$.  Recall the distribution of the iterate $z$ after one-step standard SGLD  without the accept/reject step in \eqref{eq:trans_SGLD} is
\begin{align*}
P(z|u) = \frac{1}{(4\pi\eta/\beta)^{d/2}}\exp\bigg(-\frac{\|z - u + \eta g(u)\|_2^2}{4\eta/\beta}\bigg)
\end{align*}
Since Algorithm \ref{alg:sgld} accepts the candidate only if it falls in the region $\zdom \cap \cB(u,r)$, the acceptance probability is
\begin{align*}
p(u) = \PP_{z\sim P(\cdot|u)}\big[z\in \zdom \cap \cB(u,r)\big].
\end{align*}
Therefore, the transition probability $\mathcal{T}^{\star}_{u}(z)$ for $z\in \zdom \cap \cB(u,r)$ is given by
\begin{align*}
\mathcal{T}^{\star}_u(z) = \frac{2 - p(u) + p(u)(1-\alpha_u(z))}{2} \delta_u(z) + \frac{\alpha_u(z)}{2}P(z|u)\cdot\1[z\in \zdom \cap \cB(u,r)].
\end{align*}

Take $u, v \in \zdom$ and let $\cS_u = \zdom \cap\cB(u,r)$ and $\cS_v = \zdom \cap\cB(v,r)$. By the definition of the total variation, there exists a set $\cA\in \zdom $ such that
\begin{align*}
\|\mathcal{T}^{\star}_u(\cdot) - \mathcal{T}^{\star}_v(\cdot)\|_{TV} &= |\mathcal{T}^{\star}_u(\cA) - \mathcal{T}^{\star}_v(\cA)|\notag\\
&\le  \underbrace{\max_{u,z} \bigg[\frac{2 - p(u) + p(u)(1-\alpha_u(z))}{2}\bigg]}_{I_1} \notag\\
\qquad &+ \frac{1}{2}\underbrace{\bigg|\int_{z\in\cA} \alpha_u(z)P(z|u)\1(z\in\cS_u) - \alpha_v(z)P(z|v)\1(z\in\cS_v)\d z\bigg|}_{I_2}.
\end{align*}
Using the mini-batch size that is exactly the same as the number of samples, we can reuse the bounds of $I_1$ and $I_2$ in Lemmas C.4 and C.5 of \citep{zou2020faster}. Consequently,
\begin{align*}
\|\mathcal{T}^{\star}_u(\cdot) - \mathcal{T}^{\star}_v(\cdot)\|_{TV}\le I_1 + I_2/2 \le 0.85+0.1\delta + \frac{\sqrt{\beta}\|u-v\|_2}{\sqrt{2\eta}}.
\end{align*}
By Lemma \ref{lmAuxiliaryCloseness}, we have $\delta =10Ld\eta  +10LDd^{1/2}\beta^{1/2}\eta^{3/2} \le 12Ld\eta$ if $\eta \le \frac{d}{25\beta D^2}$. Thus if 
\begin{align*}
\eta \le \frac{1}{25\beta D^2}\wedge \frac{1}{30Ld\eta} \quad \mbox{and} \quad \|u-v\|_2\le \frac{\sqrt{2\eta}}{10\sqrt{\beta}}\le 0.1 r,
\end{align*}
we have $\|\mathcal{T}^{\star}_u(\cdot) - \mathcal{T}^{\star}_v(\cdot)\|_{TV}\le 0.99$. As the result of Lemma \ref{lemma:lowerbound_transitionprob}, we prove a lower bound on the conductance $\phi$ of $\mathcal{T}^{\star}_u(\cdot)$
\begin{align*}
\phi \ge c_0\rho\sqrt{\eta/\beta},
\end{align*}
and finish the proof.

\end{proof}

\section{Property of the Gibbs algorithm}
\label{app:MarkovSemi}

\begin{proposition}\label{prop:almost_ERM-app} For $\zdom = \mathcal{B}(0, R)$, we have
\begin{align*}
	\int_{\zdom} F(z) \pi(\d z) 
	&\le \cO\left(\frac{d}{\beta}\log \frac{\beta L}{d} \right).
\end{align*}
\end{proposition}
\begin{proof}
Let $p(z) = e^{-\beta F(z)}/\Lambda$ denote the density of $\pi$. $\Lambda \deq \int_{\zdom}e^{-\beta F(z)}\d z$ is the partition function. We start by writing 
\begin{align}\label{eq:equilibrium_expected_loss}
	 \int_{\zdom} F(z) \pi(\d z) = \frac{1}{\beta}\left(h(p) - \log \Lambda\right),
\end{align}
where
$$
h(p) = -\int_{\zdom}p(z)\log p(z) \d z = -\int_{K} \frac{e^{-\beta F(z)}}{\Lambda} \log \frac{e^{-\beta F(z)}}{\Lambda} \d z
$$
is the differential entropy of $p$. To upper-bound $h(p)$, we use the fact that the differential entropy of a probability density with a finite second moment is upper-bounded by that of a Gaussian density with the same second moment. Moreover, since $p$ has the support in the Euclidean ball with radius $R$, its second moment is simply bounded by $R^2$. Therefore, we have
\begin{align}\label{eq:equilibrium_diff_entropy}
	h(p)  \le h(\mathcal{N}(0, R^2 \mathbb{I})) = \frac{d}{2}\log \frac{2\pi R^2 }{d}.
\end{align}
Next, we give a lower bound on the second term, $\log \Lambda$. We use the smoothness of $F(z)$ and the fact that $z^*$ is the minimizer of $F$. We have $F(z) \le \frac{L}{2}\| z - z^*\|^2$ for $z \in \zdom$. As such,
\begin{align}
	\log \Lambda = \log \int_{\zdom} e^{-\beta F(z)} \d z 
	&\ge \log \int_{\zdom} e^{-\beta L \| z - z^* \|^2/2} \d z \asymp O\left(\frac{d}{2}\log \frac{2\pi}{\beta L}\right). \label{eq:part_fun_LB}
\end{align}
Using \eqref{eq:equilibrium_diff_entropy} and \eqref{eq:part_fun_LB} in \eqref{eq:equilibrium_expected_loss} and simplifying, we prove the result.
\end{proof}

\end{document}